\documentclass{llncs}
\usepackage{amsmath}
\usepackage{amssymb}
\usepackage{times}
\usepackage{indentfirst}
\usepackage{graphicx}
\newtheorem{prop}{Proposition}
\newtheorem{defn}{Definition}
\newtheorem{theo}{Theorem}

\newtheorem{exa}{Example}
\newtheorem{cor}{Corollary}

\begin{document}

\title{Lattice structures of fixed points of the lower approximations of two types of covering-based rough sets}

\author{Qingyin Li and William Zhu~\thanks{Corresponding author.
E-mail: williamfengzhu@gmail.com (William Zhu)}}
\institute{
Lab of Granular Computing,\\
Zhangzhou Normal University, Zhangzhou 363000, China}



\date{\today}          
\maketitle

\begin{abstract}
Covering is a common type of data structure and covering-based rough set theory is an efficient tool to process this data.
Lattice is an important algebraic structure and used extensively in investigating some types of generalized rough sets.
In this paper, we propose two family of sets and study the conditions that these two sets become some lattice structures. These two sets
are consisted by the fixed point of the lower approximations of the first type and the sixth type of covering-based rough sets, respectively.
These two sets are called the fixed point set of neighborhoods and the fixed point set of covering, respectively.
First, for any covering, 
the fixed point set of neighborhoods is 
 a complete and distributive lattice, at the same time, it is also a double $p-$algebra. Especially, when the neighborhood forms a partition of the universe,
the fixed point set of neighborhoods is both a boolean lattice and a double Stone algebra. Second, for any covering,
the fixed point set of covering is a complete lattice.
When the covering is unary, the fixed point set of covering becomes a distributive lattice and a double $p-$algebra.
 a distributive lattice and a double $p-$algebra when the covering is unary.
Especially, when 
 the reduction of the covering forms a partition of the universe,
the fixed point set of covering is both a boolean lattice and a double Stone algebra.
\end{abstract}


\section{Introduction}

Rough set theory was introduced by Pawlak in 1982\cite{Pawlak82Rough}.
It is a new mathematical tool to handle inexact, uncertain or vague knowledge and
 has been successfully applied to various fields such as machine learning,
 pattern recognition and data mining and so on\cite{InuiguchiHiranoTsumoto03RoughSet,LinYaoZadeh01Rough,Pawlak91Rough}.
  The lower and upper approximations are two key concepts in rough set theory.
 As we know, equivalence relation or partition plays an important role in classical rough set theory.
 The family of equivalence classes forms a partition of the universe and every block of the partition is an equivalence class.
  An equivalence relation is 
  used in the definition of the lower and upper approximations.
 However, such an equivalence relation is too restrictive for many applications.
 In light of this, generalizations of rough sets were considered by many authors.
 One 
 is to consider a weaker restricted relation such as a similarity relation or tolerance
 relation\cite{Pomykala93Approximation,SlowinskiVanderpooten95Similarity}.
Another approach is 
to extend a partition to a covering.
 A covering is a more general concept than a partition.
In incomplete information systems, covering\cite{BianucciCattaneoCiucci07Entropies,QianLiangDang10Incomplete} is used to deal
with the 
attribute subset.
Since covering-based rough sets 
are more reasonable to deal with problems than classical rough sets,
who theory has obtained extensive attention and many meaningful research fruits \cite{WangZhuZhuMin12Quantitative,Zhu07Topological,ZhuWang07OnThree,ZhuWang03Reduction}. In order to establish applicable mathematical structures for
covering-based rough set and promote its applications, it has been combined
with fuzzy sets \cite{DengChenXuDai07Anovel,ZhouWu08OnGeneralized}, topology \cite{WangZhuZhu10Structure,Zhu07Topological},
and matroid \cite{LiuSai09AComparison,WangZhuMin11Transversal,ZhuWang11Matroidal}.

Lattice with both order structures and algebraic structures, and it is closely linked with many disciplines,
such as Group theory \cite{AjmalJain09Some}and so on.
Lattice theory plays an important role in many disciplines of computer
science and engineering. For example, they have applications in distributed
computing, programming language semantics \cite{MedinaOjeda10Multi,SyrrisPetridis11Alattic}.

Both lattice theory and rough sets are widely applied.
Many authors have combined these two theories and some important results
were obtained in both theoretical and application fields,
e.g. domain theory \cite{DaveyPriestley90Introduction} and Formal Concept Analysis\cite{WangLiuCao10ANew}.
 Based on the existing works 
 about the combination of rough sets and lattice theory,
Chen et al. \cite{ChenZhangYeungTsang06Rough} used the concept of covering to define approximation operators
on a complete completely distributive lattice and set up a unified framework for generalizations of rough sets.
A.A.Estaji et al. \cite{EstajiHooshmandaslDawaz12Rough} introduced the concepts of upper and lower rough
ideals and filters in a lattice, and some of their properties had been studied.
All these works help us comprehend of rough sets on lattice and have greatly enriched rough set theory and its applications.
Moreover, M.H.Ghanim et al. \cite{GhanimMustafaAziz11Onlower} defined two relations by lower intension and upper intension, respectively.
They pointed that two equivalence classes of these two relations are two partially ordered sets, respectively.
 Based on this, they studied some algebraic structures of these two  partially ordered sets.

In this paper, we study under what conditions two partially ordered sets are some lattice structures.
These two partially ordered sets are consisted by two sets together with the set inclusion, respectively.
 These two sets are called the fixed point set of neighborhoods and the fixed point set of covering, respectively.
 The fixed point set of neighborhoods is defined by the fixed points of the lower approximations of the sixth type of covering-based rough sets.
 The fixed point set of neighborhoods induced by any covering is equal to the one induced by the reduction of the covering.
 For any covering, the fixed point set of neighborhoods induced by the covering is a lattice and for any two
 elements of the fixed point set, the least upper bound is the join of these two elements and the greatest
 lower bound is the intersection of these two elements. For any covering, 
 the neighborhood of any element of the universe is a join-irreducible element of the fixed point set of neighborhoods.
 Moreover, the fixed point set of neighborhoods is also both a distributive lattice and a double $p-$algebra,
 and for any element, its pseudocomplement is the lower approximation of its complement and dual pseudocomplement is a
 union of all the neighborhoods of its complement. When the neighborhood forms a partition of the universe,
 the fixed point set of neighborhoods is both a boolean lattice and a double Stone algebra.
 The fixed point set of covering is defined by the fixed points of the lower approximations of the first type of covering-based rough sets.
 We can prove that the fixed point set of covering induced by any covering is equal to the one induced by the reduction of the covering. For any covering, the fixed point set of covering induced by the covering is a lattice and for any
 two elements of the fixed point set, the least upper bound is the join of these two elements and the greatest lower bound
is the lower approximation of the intersection of these two elements. For any covering, any irreducible element of the covering is
a join-irreducible element of the fixed point set of covering.
 The fixed point set of covering is both a distributive lattice and a double $p-$algebra when the covering is unary
 and for any element $X$, its pseudocomplement is the lower approximation of its complement and dual
pseudocomplement is a union of some blocks which are the irreducible elements of the covering and 
containing at least
an element that belongs to the complement of $X$ . When the reduction of any covering forms a partition of the universe,
the fixed point set of covering is both a boolean lattice and a double Stone algebra.

The rest of this paper is organized as follows. In Section 2, we review some basic
knowledge about rough sets and lattice. In Section 3, we study under what conditions that the fixed point set of neighborhoods becomes some lattice structures. In Section 4, we study under what conditions that the fixed point set of covering becomes some lattice structures.
Finally, we conclude this paper in Section 5.

\section{Basic definitions}

This section recalls some fundamental definitions related to rough sets and lattice.
\subsection{Covering-based rough sets}
\begin{defn} (Covering, covering approximation space\cite{BonikowskiBryniarskiWybraniecSkardowska98Extensions}) Let $U$ be a universe of discourse and $\mathbf{C}$ a family of nonempty subsets of $U$. If $\cup\mathbf{C}=U,$ then $\mathbf{C}$ is called a
covering of $U.$  We call the ordered pair $\langle U, C \rangle$ a covering approximation space.
\end{defn}

As we know, a partition of $U$ is certainly a covering of $U$, so the concept of a covering is an extension to the concept of a
partition.

For an object, we need only the essential characteristics to describing it 
rather than all characteristics. Based on this, the minimal description is established.
\begin{defn}
(Minimal description\cite{BonikowskiBryniarskiWybraniecSkardowska98Extensions,Zhu09RelationshipBetween}) Let $\langle U, C \rangle$ be a
covering approximation space, $x \in  U$. The set family $Md(x)$ is called the minimal description of $x$, where
\begin{center}
$Md(x)_{\mathbf{C}}=\{K\in \mathbf{C}|x\in K\wedge (\forall S\in \mathbf{C}\wedge x\in S \wedge S\subseteq K\Rightarrow K=S)\}.$
\end{center}
When there is no confusion, we omit the subscript $\mathbf{C}$.
\end{defn}

\begin{defn}(Unary\cite{Zhu09RelationshipBetween}) Let $\mathbf{C}$ be a covering of $U$. $\mathbf{C}$ is called unary if $\forall x\in U,$ $|Md(x)|=1$.

\end{defn}

The following theorem shows the relationship between the unary covering and the covering with the property that the intersection of
any two elements is a union of finite elements in this covering.
\begin{theo}(\cite{Zhu09RelationshipBetween})\label{T:Theo1}
A covering $\mathbf{C}$ is unary if and only if $\forall K_{1}, K_{2} \in  \mathbf{C},$ $K_{2} \cap K_{1}$ is a union of finite elements in $\mathbf{C}$.
\end{theo}

\begin{defn}(Reducible covering, irreducible covering\cite{Zhu07Topological,ZhuWang07OnThree,ZhuWang03Reduction})
Let $\mathbf{C}$ be a covering of $U$ and $K \in \mathbf{C}$. If
$K$ is a union of some sets in $\mathbf{C}-\{K\}$, we say $K$ is reducible
in $\mathbf{C}$; otherwise, $K$ is irreducible. If every element of $\mathbf{C}$ is an irreducible
element, we say $\mathbf{C}$ is irreducible; otherwise $\mathbf{C}$  is reducible.
\end{defn}

The following two propositions are important results for us to apply the reducible element concept to covering-based
rough sets.
\begin{prop}(\cite{Zhu09RelationshipBetween})
Let $\mathbf{C}$ be a covering of $U$. If $K$ is a reducible element of
$\mathbf{C}$, then $\mathbf{C}-\{K\}$ is still a covering of $U$.
\end{prop}
\begin{prop}(\cite{Zhu09RelationshipBetween})
Let $\mathbf{C}$ be a covering of $U$, $K \in \mathbf{C}$, $K$ is a reducible element of $\mathbf{C}$,
and $K_{1} \in \mathbf{C}-\{K\}$, then $K_{1}$ is a reducible element of $\mathbf{C}$ if and only if it is a reducible element of $\mathbf{C}-\{K\}$.
\end{prop}

\begin{defn}(A reduction of covering\cite{Zhu09RelationshipBetween})
For a covering $\mathbf{C}$ of $U$, the new irreducible covering
through the above reduction is called the reduction of $\mathbf{C}$, and denoted by
$reduct(\mathbf{C})$ .
\end{defn}
\begin{defn}(Neighborhood\cite{BonikowskiBryniarskiWybraniecSkardowska98Extensions,Zhu09RelationshipBetween})
Let $\mathbf{C}$ be a covering of $U$ and $x \in U$. $N_{\mathbf{C}}(x)=\cap\{K\in \mathbf{C}|x\in K\}$
is called the neighborhood of $x$ with respect to $\mathbf{C}$. When
there is no confusion, we omit the subscript $\mathbf{C}$.
\end{defn}

Our investigation in this paper involves two types of covering-based rough sets. They are the first and sixth types of covering-based rough sets.
\begin{defn} \label{D:DefnCoveringApproximation}
 (The first type of covering-based rough sets\cite{Zhu09RelationshipBetween}) Let $\mathbf{C}$ be a covering of $U.$ For any $X\subseteq U,$
\begin{center}
$FL_{\mathbf{C}}(X)=\cup\{K\in \mathbf{C}|K\subseteq X\},$\\
$FH_{\mathbf{C}}(X)=\cup\{K\in \mathbf{C}|K\cap X\neq \emptyset\},$
\end{center}
are called the first type of covering lower and upper approximations of $X$, respectively. When there is no confusion, we omit $\mathbf{C}$ at the lowercase.
\end{defn}

\begin{defn}(The sixth type of covering-based rough sets\cite{Zhu09RelationshipBetween})\label{D:DefnNeighborhoodApproximation}
Let $\mathbf{C}$ be a covering of $U.$ For any $X\subseteq U,$
\begin{center}
$XL_{\mathbf{C}}(X)=\{x\in U|N(x)\subseteq X\},$\\
$XH_{\mathbf{C}}(X)=\{x\in U|N(x)\cap X\neq \emptyset\},$
\end{center}
are called the sixth type of covering lower and upper approximations of $X$, respectively. When there is no confusion, we omit $\mathbf{C}$ at the lowercase.
\end{defn}

\begin{prop}(\cite{ZhuWang07OnThree,ZhuWang03Reduction})\label{P:Prop3}
Let $\mathbf{C}$ be a covering of $U$ and $\emptyset$ the empty set. For any $X\subseteq U,$ both the first and the sixth types of covering-based rough sets have the following properties:\\
   $(1)$ $FL(\emptyset)=\emptyset$, $XL(\emptyset)=\emptyset$;\\
   $(2)$ $FL(U)=U$, $XL(U)=U$;\\
   $(3)$ $FL(X)\subseteq X$, $XL(X)\subseteq X$;\\
   $(4)$ $FL(FL(X))=FL(X)$, $XL(XL(X))=XL(X)$;\\
   $(5)$ $ X\subseteq Y \Rightarrow FL(X)\subseteq FL(Y)$, $XL(X)\subseteq XL(Y)$;\\
   $(6)$ $\forall K\in \mathbf{C}, FL(K)=K$, $XL(K)=K$.
\end{prop}
\begin{prop}(\cite{ZhuWang03Reduction})\label{P:Prop4}
Let $\mathbf{C}$ be a covering of $U$ and $K$ a reducible element of $\mathbf{C}$. For any $X\subseteq U$, $FL_{\mathbf{C}}(X)=FL_{\mathbf{C}-\{K\}}(X)$.
\end{prop}

The above proposition shows that $\mathbf{C}$ and $\mathbf{C}-\{K\}$ generate the same first type of lower approximation. The following corollary shows that $\mathbf{C}$ and $reduct(\mathbf{C})$ generate the same first type of lower approximation.
\begin{cor}(\cite{ZhuWang03Reduction})\label{C:Cor1}
Let $\mathbf{C}$ be a covering of $U$. For any $X\subseteq U$, $FL_{reduct(\mathbf{C})}(X)=FL_{\mathbf{C}}(X)$.
\end{cor}

The following theorem shows $\mathbf{C}$ and $reduct(\mathbf{C})$ generate the same sixth type of lower approximation.
\begin{theo}(\cite{Zhu09RelationshipBetween})\label{T:Theo2}
Let $\mathbf{C}$ be a covering of $U$. For any $X\subseteq U$, $XL_{reduct(\mathbf{C})}(X)=XL_{\mathbf{C}}(X)$.
\end{theo}
\subsection{Partially ordered set and lattice}

\begin{defn}(Partially ordered set\cite{Birkhoff95Lattice,GengQuWang04Discrete,Gratzer78General,Gratzer71Lattice})
Let $P$ be a nonempty set and $\leq $ a partial order on $P.$ For any $x, y, z \in P,$ if \\
$(1)$ $x\leq x$;\\
$(2)$ $x\leq y$ and $y\leq x$ imply $x=y$;\\
$(3)$ $x\leq y$ and $y \leq z$ imply $x\leq z$,\\
 then $\langle P,\leq\rangle$ (or $P$ for short) is called a partially ordered set.
\end{defn}

Based on the partially ordered set, we introduce the concept of lattice.
\begin{defn}(\cite{Birkhoff95Lattice,GengQuWang04Discrete,Gratzer78General,Gratzer71Lattice})\label{D:Defnlattice}
 A partially ordered set $\langle P,\leq\rangle$ is a lattice if $a\vee b$ and $a\wedge b$ exist for all $a,b \in P$.  $\langle P,\vee,\wedge\rangle$ is called an algebraic system induced by lattice $\langle P,\leq\rangle$.
 \end{defn}

In the following, we list the properties of the algebraic system induced by a lattice.
\begin{theo}(\cite{Birkhoff95Lattice,GengQuWang04Discrete,Gratzer78General,Gratzer71Lattice})\label{T:TheoIdentities}
Let $\langle P,\vee,\wedge\rangle$ be an algebraic system induced by lattice $\langle P,\leq\rangle.$ For any $a,b,c \in P,$ the algebra has the following identities:\\
$(P1)$ $a\vee a=a,a\wedge a=a;$\\
$(P2)$ $a\vee b=b\vee a,a\wedge b=b \wedge a;$\\
$(P3)$ $(a\vee b)\vee c=a\vee (b \vee c);$\\
$(P4)$ $a\vee (a\wedge b)=a, a\wedge(a\vee b)=a$.
\end{theo}

The following definition shows the conditions that an algebraic system is a lattice. 
\begin{defn}(\cite{Birkhoff95Lattice,GengQuWang04Discrete,Gratzer78General,Gratzer71Lattice})\label{D:Defnalgebralattice}
Let $\langle P,\vee,\wedge\rangle$ be an algebraic system. 
If $\wedge$ and $\vee$ 
satisfy $(P2)-(P4),$ then $\langle P,\vee,\wedge\rangle$ is a lattice.
\end{defn}

According to Definition~\ref{D:Defnlattice}, Theorem~\ref{T:TheoIdentities} and Definition~\ref{D:Defnalgebralattice}, $\langle P,\vee,\wedge\rangle$  is an algebraic system induced by $\langle P,\leq\rangle$ and $\langle P,\vee,\wedge\rangle$ is also a lattice.
Therefore, we no longer differentiate $\langle P,\vee,\wedge\rangle$ and $\langle P,\leq\rangle$ and both of them are called lattice $P$.

Several special types of lattices are introduced in the following five definitions. 
\begin{defn}(Complete lattice\cite{Birkhoff95Lattice,GengQuWang04Discrete,Gratzer78General,Gratzer71Lattice})
A lattice $P$ is a complete lattice if $~\forall S\subseteq P,$ $\wedge S$ and $\vee S$ both in $P$.
\end{defn}

\begin{defn}(Distributive lattic\cite{Birkhoff95Lattice,GengQuWang04Discrete,Gratzer78General,Gratzer71Lattice})
A lattice $P$ is a distributive lattice if
\begin{center}
$a\vee(b\wedge c)=(a\vee b)\wedge (a\vee c)~~~~$  or  $~~~~~a\wedge (b\vee c)=(a\wedge b)\vee (a\wedge c)$
\end{center}
for all $a,b,c\in P$.
\end{defn}

\begin{defn}(Bounded lattice\cite{Birkhoff95Lattice,GengQuWang04Discrete,Gratzer78General,Gratzer71Lattice})
Let $\langle P,\vee,\wedge\rangle$ be a lattice. We say that $P$ has a greatest element if there exists an element $1 \in L$ such that $a\wedge 1=a$ for all $a\in P$. Dually, $P$ is said to have a least element if there exists an element $0\in L$ such that $a\vee 0=a$ for all $a \in P$. A lattice $\langle P,\vee,\wedge \rangle$ possessing $0$ and $1$ is called a bounded lattice.
\end{defn}

\begin{defn}(Complemented lattice\cite{Birkhoff95Lattice,GengQuWang04Discrete,Gratzer78General,Gratzer71Lattice})
Let $P$ be a bounded lattice with a least element $0$ and a greatest element $1$. For an element $a\in P$, we say that an element $b\in P$ is a complement of $a$ if $a\vee b =1$ and $a \wedge b =0$. If the element $a$ has a unique complement, we denote it by $a^{c}.$ A lattice $P$ is a complemented lattice if each element has a complement.
\end{defn}

\begin{defn}(Boolean lattice\cite{Birkhoff95Lattice,GengQuWang04Discrete,Gratzer78General,Gratzer71Lattice})
A lattice $P$ is called a boolean lattice if it is a complemented and distributive lattice.
\end{defn}

\begin{defn}(Join-irreducible\cite{Birkhoff95Lattice,GengQuWang04Discrete,Gratzer78General,Gratzer71Lattice})
Let $P$ be a lattice. An element $a \in P$ is called a join-irreducible if
 $a=b\vee c$ implies $a=b $ or $a = c$ for all $b, c\in P$. And all join-irreducible elements in $P$ are denoted by $\mathcal{J}(P)$
\end{defn}

Several definitions relate to Stone algebra are introduced in the following.
\begin{defn}(Stone algebra\cite{Katrinak74Injection,Katrinak74Construction})\label{D:Stone}
Let $P$ be a lattice with a least element. An element $x^{\ast}$ is a pseudocomplement of $x \in P$,
if $x \wedge x^{\ast}=0$ and for all $y\in P$, $x\wedge y=0$ implies $y \leq x^{\ast}$.
A lattice is pseudocomplemented if each element has a pseudocomplement.
 If $P$ is a distributive pseudocomplemented lattice, and it satisfies the Stone identity
$x^{\ast}\vee x^{\ast\ast} =1$ for all $x\in P$, then $P$ is called a Stone algebra.
\end{defn}

\begin{defn}(Dual Stone algebra\cite{Katrinak74Injection,Katrinak74Construction})\label{D:DualStone}
Let $P$ be a lattice with a greatest element. An element $x^{+}$ is a dual pseudocomplement of $x\in P$, if $x\vee x^{+}=1$ and for all $y \in P$, $x \vee y =1$ implies $x^{+}\leq y$. A lattice is dual pseudocomplemented if each element has a dual pseudocomplement. If $P$ is a distributive dual pseudocomplemented lattice, and it satisfies the dual Stone identity $x^{+}\wedge x^{++}=0$ for all $x \in P$, then $P$ is called a dual Stone algebra.
\end{defn}


\begin{defn}(Double $p-$algebra\cite{Katrinak74Injection,Katrinak74Construction})\label{D:Palgebra}
A lattice $P$ is called a double $p-$algebra if it is pseudocomplemented and dual pseudocomplemented.
\end{defn}

\begin{defn}(Double Stone algebra\cite{Katrinak74Injection,Katrinak74Construction})
A lattice $P$ is called a double Stone algebra if it is a Stone and a dual Stone algebra.
\end{defn}

\section{Lattice structure of the fixed point of the lower approximations of neighborhoods}
In this section, 
we study under what conditions a set becomes some special lattices, where the set is consisted of
the fixed points of lower approximations of the sixth type of covering-based rough sets. 
\begin{defn}
Let $\mathbf{C}$ be a covering of $U$. We define
\begin{center}
 $\mathcal{P}_{\mathbf{C}}=\{X\subseteq U|XL_{\mathbf{C}}(X)=X\}.$
 \end{center}
 $\mathcal{P}_{\mathbf{C}}$ is called the fixed point set of neighborhoods induced by $\mathbf{C}$.
We omit the subscript $\mathbf{C}$ when there is no confusion.
\end{defn}

 The following proposition shows that the fixed point set of neighborhoods induced by any covering
 of the universe is equal to the one induced by the reduction of the covering. 
\begin{prop}　　
 Suppose $\mathbf{C}$ is a covering of $U$, then $\mathcal{P}_{\mathbf{C}}=\mathcal{P}_{reduct(\mathbf{C})}.$
\end{prop}
\begin{proof}
According to Definition~\ref{D:DefnNeighborhoodApproximation}, $\mathcal{P}_{\mathbf{C}}=\{X\subseteq U|XL_{\mathbf{C}}(X)=X\}$ and $\mathcal{P}_{reduct(\mathbf{C})}=\{X\subseteq U|XL_{reduct(\mathbf{C})}(X)=X\}.$
According to Theorem~\ref{T:Theo2}, $XL_{reduct(\mathbf{C})}(X)=XL_{\mathbf{C}}(X)$ for any $X\subseteq U,$
Thus $\mathcal{P}_{\mathbf{C}}=\mathcal{P}_{reduct(\mathbf{C})}.$
\end{proof}

For any covering $\mathbf{C}$ of $U$, the fixed point set of neighborhoods together with the set inclusion, 
 $\langle \mathcal{P}, \subseteq \rangle,$ is a partially ordered set.

The following proposition presents an equivalent characterization of the element of the fixed point set of neighborhoods.
\begin{prop}
$X\in \mathcal{P}$ iff $X=\cup_{x\in X} N(x)$.
\end{prop}
\begin{proof}
For any $X\in \mathcal{P},$ $XL(X)=X.$ Since $XL(X)=\{x\in U|N(x)\subseteq X\}=X$, $N(x)\subseteq X$ for all $x\in X.$ Therefore $\cup_{x\in X} N(x)\subseteq X$. Since $x\in N(x),$ then $X\subseteq \cup_{x\in X} N(x)$. so $X=\cup_{x\in X} N(x).$ Conversely, if $\cup_{x\in X} N(x)= X$, then $N(x)\subseteq X$ for all $x\in X$. For any $y\notin X,$ $y\in N(y)\nsubseteq X$. Hence $XL(X)=\{x\in U|N(x)\subseteq X\}=X$, i.e., $X\in \mathcal{P}.$
\end{proof}

As we know, $\langle \mathcal{P}, \subseteq \rangle$ is a partially ordered set. Naturally we consider that whether this partially ordered set is a lattice. In the following, we investigate lattice structures of this partially ordered set.
\begin{theo}\label{T:Lattice1}
$\langle \mathcal{P}, \subseteq \rangle$ is a lattice, where 
 $X\vee Y =X\cup Y$ and 
  $X\wedge Y= X\cap Y$ for any $X,Y\in \mathcal{P}.$
\end{theo}
\begin{proof}
For any $X, Y\in\mathcal{P},$ if $X\cup Y\notin \mathcal{P},$ then there exists $x\in X\cup Y$ such that $N(x)\nsubseteq X\cup Y.$ Since $x\in X\cup Y,$ $x\in X$ or $x\in Y.$ Hence $N(x)\nsubseteq X$ or $N(x)\nsubseteq Y$, which is contradictory with $X, Y\in \mathcal{P}.$ Therefore, $X\cup Y\in \mathcal{P}.$

For any $X, Y\in\mathcal{P},$ if $X\cap Y\notin \mathcal{P},$ then there exists $y\in X\cap Y$ such that $N(x)\nsubseteq X\cap Y.$ Since $x\in X\cap Y,$ $x\in X$ and $x\in Y.$ Hence there exist three cases as follows: $(1)$ $N(y)\nsubseteq X$ and $N(y)\nsubseteq Y$, $(2)$ $N(y)\nsubseteq X$ and $N(y)\subseteq Y$, $(3)$ $N(y)\subseteq X$ and $N(y)\nsubseteq Y$. But these three cases are all contradictory with $X, Y\in \mathcal{P}.$ Therefore, $X\cap Y\in \mathcal{P}.$ Thus $\langle \mathcal{P}, \subseteq \rangle$ is a lattice.
\end{proof}
\begin{remark}
$\emptyset$ and $U$ are the least and greatest elements of $\langle \mathcal{P}, \subseteq \rangle,$ respectively. Therefore, $\langle \mathcal{P}, \subseteq \rangle$ is a bounded lattice.
\end{remark}

Theorem~\ref{T:Lattice1} shows that the fixed point set of neighborhoods together with the set inclusion is a lattice,
and for any two elements of the fixed point set, the least upper bound is the join of these two elements and the greatest lower
bound is the intersection of these two elements.
In fact, $\langle \mathcal{P}, \cap, \cup \rangle$
is defined from the viewpoint of algebra
and $\langle \mathcal{P}, \subseteq \rangle$
is defined from the viewpoint of partially ordered set.
Both of them are lattices. 
Therefore, we no longer
 differentiate $\langle \mathcal{P}, \cap, \cup \rangle$ and $\langle \mathcal{P}, \subseteq \rangle$,
 and both of them are called lattice $\mathcal{P}$.

The following proposition shows that the neighborhood of any element of the universe 
belongs to the fixed point set of neighborhoods.
\begin{prop}
Let $\mathbf{C}$ be a covering of $U$. For all $x\in U,$  $N(x)\in \mathcal{P}.$
\end{prop}
\begin{proof}
For any $y\in N(x)$, $N(y)\subseteq N(x),$ which implies $y\in \{z|N(z)\subseteq N(x)\}=XL(N(x)).$ Hence $N(x)\subseteq XL(N(x)).$ According to Proposition~\ref{P:Prop3}, $XL(N(x))\subseteq N(x).$ Thus $XL(N(x))=N(x),$ i.e., $N(x)\in \mathcal{P}.$
\end{proof}

The following proposition points out that the neighborhood of any element of the universe 
is a join-irreducible element of the fixed point set of neighborhoods.
\begin{prop}
Let $\mathbf{C}$ be a covering of $U$. For any $x\in U,$ $N(x)$ is a join-irreducible element of the lattice $\mathcal{P}$.
\end{prop}
\begin{proof}
 Suppose there exist $X, Y\in \mathcal{P}$ such that $N(x)=X\cup Y.$ Since $x\in N(x),$ $x\in X\cup Y.$ Therefore, $x\in X$ or $x\in Y.$ Moreover, since $X, Y\in \mathcal{P},$ then $N(x)\subseteq X\subseteq X\cup Y=N(x)$ or $N(x)\subseteq Y\subseteq X\cup Y=N(x).$ Therefore, $N(x)=X$ or $N(x)=Y.$
 Thus $N(x)$ is a join-irreducible element of the lattice $\mathcal{P}$ for all $x\in U.$
\end{proof}

According to Theorem~\ref{T:Lattice1}, the fixed point set of neighborhoods induced by any covering is a lattice. In fact, it is also a complete lattice.
\begin{theo}
Let $\mathbf{C}$ be a covering of $U$. 
$\mathcal{P}$ is a complete lattice.
\end{theo}
\begin{proof}
For any $\mathcal{S}\subseteq \mathcal{P},$ we need to prove that $\cap \mathcal{S}\in \mathcal{P}$ and $\cup \mathcal{S}\in \mathcal{P}.$

If $\cap \mathcal{S}\notin \mathcal{P},$ then there exists $y\in \cap \mathcal{S}$ such that $N(y)\nsubseteq\cap \mathcal{S},$ i.e., there are two index sets $I, J\subseteq \{1,2,\cdots, |\mathcal{S}|\}$ with $I\cap J=\emptyset$ and $|I\cup J|=|\mathcal{S}|$ such that $N(y)\nsubseteq X_{i}$ and $N(y)\subseteq X_{j}$ for any $i\in I, j\in J,$ where $X_{i}, X_{j}\in \mathcal{S}.$ This is contradictory with $X_{i}(i\in I), X_{j}(j\in J)\in \mathcal{P}.$ Hence $\cap \mathcal{S}\in \mathcal{P}.$

If $\cup \mathcal{S}\notin \mathcal{P},$ then there exists $x\in \cup \mathcal{S}$ such that $N(x)\nsubseteq\cup \mathcal{S},$ i.e., there exists $X\in \mathcal{S}$ such that $x\in X$ and $N(x)\nsubseteq X,$ which is contradictory with $X\in \mathcal{P}.$ Hence $\cup \mathcal{S}\in \mathcal{P}.$
\end{proof}

The following theorem shows that the fixed point set of neighborhoods induced by any covering is a distributive lattice.

\begin{theo}\label{T:Distributive1}
Let $\mathbf{C}$ be a covering of $U$. 
$\mathcal{P}$ is a distributive lattice.
\end{theo}
\begin{proof}
For any $X, Y, Z\in \mathcal{P}$, $X, Y, Z\subseteq U.$  It is straightforward that $X\cup(Y\cap Z)=(X\cup Y)\cap(X\cup Z), X\cap(Y\cup Z)=(X\cap Y)\cup(X\cap Z).$ Hence $\mathcal{P}$ is a distributive lattice.
\end{proof}

The fixed point set of neighborhoods induced by any covering is both a pseudocomplemented and a dual pseudocomplemented lattice. That is to say any element of the fixed point set of neighborhoods has a pseudocomplement and a dual pseudocomplement. For any element, its pseudocomplement is the lower approximation of its complement and dual pseudocomplement is the union of all the neighborhood of its complement.

\begin{theo}\label{T:Stone1}
Let $\mathbf{C}$ be a covering of $U$. Then:\\
$(1)$ 
$\mathcal{P}$ is a pseudocomplemented lattice, and $X^{\ast}=XL(X^{c})$ for any $X\in \mathcal{P};$\\
$(2)$ 
$\mathcal{P}$ is a dual pseudocomplemented lattice, and $X^{+}=\cup_{x\in X^{c}}N(x)$ for any $X\in \mathcal{P}.$\\
Where $X^{c}$ is the complement of $X$ in $U$.
\end{theo}
\begin{proof}
$(1)$ For any $X\in \mathcal{P}$, according to Proposition~\ref{P:Prop3}, $XL(XL(X^{c}))=XL(X^{c})$,
then $XL(X^{c})\in \mathcal{P}.$ According to Proposition~\ref{P:Prop3}, $XL(X^{c})\subseteq X^{c}$. Hence $X\cap XL(X^{c})=\emptyset.$
For any $Y\in \mathcal{P},$ if $X\cap Y=\emptyset,$ then $Y\subseteq X^{c}.$ According to Proposition~\ref{P:Prop3}, $Y=XL(Y)\subseteq XL(X^{c}).$
Therefore,  $X^{\ast}=XL(X^{c})$ for any $X\in \mathcal{P},$ i.e.,
$\mathcal{P}$ is a pseudocomplemented lattice.

$(2)$ First, we prove $\cup_{x\in X^{c}}N(x)\in\mathcal{P}$ for any $X\in \mathcal{P}.$ For any $y\in \cup_{x\in X^{c}}N(x),$
there exists $z\in  X^{c}$ such that $y\in N(z).$ Thus $N(y)\subseteq N(z),$ i.e., $N(y)\subseteq \cup_{x\in X^{c}}N(x).$ Therefore, $y\in XL(\cup_{x\in X^{c}}N(x)),$ i.e., $\cup_{x\in X^{c}}N(x)\subseteq XL(\cup_{x\in X^{c}}N(x)).$ According to Proposition~\ref{P:Prop3}, $XL(\cup_{x\in X^{c}}N(x))\subseteq\cup_{x\in X^{c}}N(x).$ Consequently, $XL(\cup_{x\in X^{c}}N(x))=\cup_{x\in X^{c}}N(x),$ i.e., $\cup_{x\in X^{c}}N(x)\in\mathcal{P}.$

It is straightforward that $X\cup(\cup_{x\in X^{c}}N(x))=U.$

Second, we need to prove that for any $Y\in \mathcal{P},$ if $X\cup Y=U,$ then $\cup_{x\in X^{c}}N(x)\subseteq Y.$
 The following two cases are used to prove it. Case 1: If $\cup_{x\in X^{c}}N(x)=X^{c},$ then $\cup_{x\in X^{c}}N(x)\subseteq Y.$
 Case 2: If $X^{c}\subset\cup_{x\in X^{c}}N(x),$ then $X^{c}\subset Y.$
 Suppose $Y\subset\cup_{x\in X^{c}}N(x),$
 then there exists $y\in \cup_{x\in X^{c}}N(x)$ such that $y\notin Y,$ so $y\notin X^{c},$ which implies there exists $z\in X^{c},$
 such that $y\in N(z).$ Since $X^{c}\subset Y,$ $z\in Y.$ So $N(z)\nsubseteq Y,$ i.e.,$z\notin XL(Y).$ In other words, $XL(Y)\neq Y,$
 which is contradictory with $Y\in \mathcal{P}.$ Hence $\cup_{x\in X^{c}}N(x)\subseteq Y.$
 Consequently, $X^{+}=\cup_{x\in X^{c}}N(x)$ for any $X\in \mathcal{P},$ i.e.,
$\mathcal{P}$ is a dual pseudocomplemented lattice.
\end{proof}

According to Theorem~\ref{T:Stone1}, Definition~\ref{D:Stone} and Definition~\ref{D:DualStone}, the fixed point set of neighborhoods induced by any covering is both a pseudocomplemented and a dual pseudocomplemented lattice. Moreover, according to Definition~\ref{D:Palgebra}, it is a double $p-$algebra.
\begin{remark}
Generally, the fixed point set of neighborhoods neither a Stone algebra nor a dual Stone algebra. 
\end{remark}
\begin{exa}\label{E:Example1}
Let $U=\{1,2,3,4\}$ and $\mathbf{C}=\{\{1,2,3\},\{1\},\{1,3,4\},\{2,3\}\}.$
Then $N(1)=\{1\}, N(2)=\{2,3\}, N(3)=\{3\}, N(4)=\{1,3,4\}.$ If $X=\{2, 3\},$
then $X^{\ast}=XL(X^{c})=\{1\}, X^{\ast\ast}=XL((X^{\ast})^{c})=\{2,3\},$ i.e.,
$X^{\ast}\cup X^{\ast\ast}\neq U.$ Therefore,
$\mathcal{P}$ is not a Stone algebra. Similarly,
$X^{+}=\cup_{x\in X^{c}}N(x)=\{1,3,4\}, X^{++}=\cup_{y\in (X^{+})^{c}}N(y)=\{2,3\},$
i.e., $X^{+}\cap X^{++}\neq \emptyset.$ Thus
$\mathcal{P}$ is not a dual Stone algebra.
\end{exa}

According to Example~\ref{E:Example1}, the fixed point set of neighborhoods induced by any covering is not always a double Stone algebra. In the following, we study under what conditions that the fixed point set of neighborhoods induced by a covering is a boolean lattice and a double Stone algebra, respectively.
\begin{theo}\label{T:Boolean1}
If $\{N(x)|x\in U\}$ is a partition of $U,$ then 
$\mathcal{P}$ is a boolean lattice.
\end{theo}
\begin{proof}
According to Theorem~\ref{T:Distributive1},
 $\mathcal{P}$ is a distributive lattice. Moreover,
 $\mathcal{P}$ is a bounded lattice. In the following,
 we need to prove only that 
 $\mathcal{P}$ is a complemented lattice. In other words,
 we need to prove that $X^{c}\in \mathcal{P}$ for any $X\in \mathcal{P}.$ If $X^{c}\notin \mathcal{P},$
 i.e., $\cup_{x\in X^{c}}N(x)\neq X^{c},$
 then there exists $y\in \cup_{x\in X^{c}}N(x)$ such that $y\notin X^{c}.$ Since $y\in \cup_{x\in X^{c}}N(x)$, then there exists $z\in X^{c}$ such that $y\in N(z).$
 Since $\{N(x)|x\in U\}$ is a partition of $U,$
 $z\in N(y).$
 Therefore, $N(y)\nsubseteq X,$ i.e., $\cup_{x\in X}N(x)\neq X,$ which is contradictory with $X\in \mathcal{P}.$ Hence, $X^{c}\in \mathcal{P}$ for any $X\in \mathcal{P},$ i.e.,
 $\mathcal{P}$ is a complemented lattice. Consequently,
 $\mathcal{P}$ is a boolean lattice.
\end{proof}

\begin{theo}\label{T:DoubleStone1}
If $\{N(x)|x\in U\}$ is a partition of $U,$ then
$\mathcal{P}$ is a double Stone algebra.
\end{theo}
\begin{proof}
For any $X\in \mathcal{P}$, we prove $X^{\ast}=X^{c}=X^{+}.$ 
Suppose for any $y\in X^{c}$ there exists $z\in X$ such that $z\in N(y),$ i.e., $N(y)\nsubseteq X^{c}$. Since $\{N(x)|x\in U\}$ is a partition of $U,$
then $y\in N(z),$ i.e., $N(z)\nsubseteq X.$ So $z\notin XL(X),$ which is contradictory with $X\in \mathcal{P}.$ Hence $N(y)\subseteq X^{c}.$
Then $y\in XL(X^{c})$ and $\cup_{x\in X^{c}}N(x)\subseteq X^{c},$ i.e., $X^{c}\subseteq XL(X^{c}).$
 According to Proposition~\ref{P:Prop3}, $XL(X^{c})\subseteq X^{c}.$
 It is straightforward that $X^{c}\subseteq \cup_{x\in X^{c}}N(x).$
 Consequently, $X^{\ast}=XL(X^{c})=X^{c}=\cup_{x\in X^{c}}N(x)=X^{+}.$
  Since $XL(X^{c})=X^{c},$ then $X^{c}\in \mathcal{P}.$ Similarly, we can prove that $X^{\ast\ast}=X^{cc}=X=X^{++}.$
   Therefore, $X^{\ast}\cup X^{\ast\ast}=U, X^{+}\cap X^{++}=\emptyset,$ i.e.,
$\mathcal{P}$ is both a Stone and a dual Stone algebra. Consequently,
$\mathcal{P}$ is a double Stone algebra.
\end{proof}

According to Theorem~\ref{T:Boolean1} and Theorem~\ref{T:DoubleStone1}, the fixed point set of neighborhoods is both a boolean lattice and a double Stone algebra when the neighborhood forms a partition of the universe.

\section{Lattice structure of the fixed point of the lower approximations of covering}
In this section, we propose a family of sets which is consisted of the fixed points of lower approximations of the first type of covering-based rough sets
 and study the conditions when the family of sets becomes some special lattices.
\begin{defn}\label{D:DefnFixedPointOfCovering}
Let $\mathbf{C}$ be a covering of $U$. We define
\begin{center}
 $\mathcal{F}_{\mathbf{C}}=\{X\subseteq U|FL_{\mathbf{C}}(X)=X\}.$
 \end{center}
 $\mathcal{F}_{\mathbf{C}}$ is called the fixed point set of covering induced by $\mathbf{C}$.
We omit the subscript $\mathbf{C}$ when there is no confusion.
\end{defn}

Reducible element is an important concept in covering-based rough sets.
It is interesting to consider the influence of reducible elements on the fixed point set of covering.
 As shown in the following proposition, the fixed point set of covering induced by any covering $\mathbf{C}$ 
is equal to the one induced by the covering $\mathbf{C}-\{K\},$ if 
$K$ is a reducible element of $\mathbf{C}$.
\begin{prop}　　
 Suppose $\mathbf{C}$ is a covering of $U$ and $K$ is a reducible element of $\mathbf{C}$, then $\mathcal{F}_{\mathbf{C}}=\mathcal{F}_{\mathbf{C}-\{K\}}.$
\end{prop}
\begin{proof}
According to Definition~\ref{D:DefnFixedPointOfCovering}, $\mathcal{F}_{\mathbf{C}}=\{X\subseteq U|FL_{\mathbf{C}}(X)=X\}$ and $\mathcal{F}_{\mathbf{C}-\{K\}}=\{X\subseteq U|FL_{\mathbf{C}-\{K\}}(X)=X\}.$
According to Proposition~\ref{P:Prop4}, $FL_{\mathbf{C}-\{K\}}(X)=FL_{\mathbf{C}}(X)$ for any $X\subseteq U.$
Thus $\mathcal{F}_{\mathbf{C}}=\mathcal{F}_{\mathbf{C}-\{K\}}.$
\end{proof}

The fixed point set of covering induced by any covering $\mathbf{C}$ 
is equal to the one induced by the reduction of the covering. 
\begin{cor}　　
 Suppose $\mathbf{C}$ is a covering of $U$, then $\mathcal{F}_{\mathbf{C}}=\mathcal{F}_{reduct(\mathbf{C})}.$
\end{cor}
\begin{proof}
According to Definition~\ref{D:DefnFixedPointOfCovering}, $\mathcal{F}_{\mathbf{C}}=\{X\subseteq U|FL_{\mathbf{C}}(X)=X\}$ and $\mathcal{F}_{reduct(\mathbf{C})}=\{X\subseteq U|FL_{reduct(\mathbf{C})}(X)=X\}.$
According to Corollary~\ref{C:Cor1}, $FL_{reduct(\mathbf{C})}(X)=FL_{\mathbf{C}}(X)$ for any $X\subseteq U.$
Thus $\mathcal{F}_{\mathbf{C}}=\mathcal{F}_{reduct(\mathbf{C})}.$
\end{proof}

For any covering $\mathbf{C}$ of $U$, the fixed point set of covering together with the set inclusion,
$\langle \mathcal{F}, \subseteq \rangle,$ is a partially ordered set.
Naturally we consider that whether this partially ordered set is a lattice. The following theorem shows
that the fixed point set of covering is a lattice, and for any two elements of this lattice, the join of these two
elements is the least upper bound of this lattice and the lower approximation of the intersection of these two elements is the greatest
 lower bound of this lattice.
\begin{theo}\label{T:Lattice2}
$\langle \mathcal{F}, \subseteq \rangle$ is a lattice, where 
$X\vee Y =X\cup Y$ and 
$X\wedge Y= FL(X\cap Y)$ for any $X,Y\in \mathcal{F}.$
\end{theo}
\begin{proof}
We need to prove only that $X\cup Y\in \mathcal{F}$ and $FL(X\cap Y)\in \mathcal{F}$ for any $X, Y\in\mathcal{F}.$

For any $X, Y\in\mathcal{F},$ if $X\cup Y\notin \mathcal{F},$ then there exists $x\in X\cup Y$ such that $\forall K\in \mathbf{C}$ and $x\in K$ implies $K\nsubseteq X\cup Y.$ Since $x\in X\cup Y,$ $x\in X$ or $x\in Y.$ Hence $K\nsubseteq X$ or $K\nsubseteq Y$, i.e., $x\notin FL(X)$ or $x\notin FL(Y),$ which are contradictory with $X, Y\in \mathcal{F}.$ Therefore, $X\cup Y\in \mathcal{F}.$

For any $X, Y\in\mathcal{F},$  $X\cap Y\subseteq U.$
 According to Proposition~\ref{P:Prop3}, $FL(FL(X\cap Y))=FL(X\cap Y).$ Hence $FL(X\cap Y)\in \mathcal{F}.$ Thus $\langle \mathcal{F}, \subseteq \rangle$ is a lattice.
\end{proof}

Theorem~\ref{T:Lattice2} shows that the fixed point set of covering together with the set inclusion is a lattice, and for any two elements of the fixed point set of covering,
the least upper bound is the join of these two elements and the greatest lower bound is the 
lower approximation of the intersection of these two elements.
In fact, $\langle \mathcal{F}, \wedge, \vee \rangle$
is defined from the viewpoint of algebra
and $\langle \mathcal{F}, \subseteq \rangle$ is
defined from the viewpoint of  partially ordered set.
Both of them are lattices. 
Therefore, we no longer
 differentiate $\langle \mathcal{F}, \wedge, \vee \rangle$ and $\langle \mathcal{F}, \subseteq \rangle$,
 and both of them are called lattice $\mathcal{F}$.
\begin{remark}
$\emptyset$ and $U$ are the least and greatest elements of $\mathcal{F},$ respectively. Therefore, $ \mathcal{F}$ is a bounded lattice.
\end{remark}

The following proposition points out that each irreducible element of a covering  
is a join-irreducible element of the fixed point set of covering, and any reducible element of the covering 
is a join-reducible element of the fixed point set of covering.
\begin{prop}\label{P:Prop10}
Let $\mathbf{C}$ be a covering of $U$. For any $K\in \mathbf{C}$, \\
$(1)$ $K$ is a join-irreducible element of $\mathcal{F}$ if $K$ is an irreducible element of $\mathbf{C},$ \\
$(2)$ $K$ is a join-reducible element of $\mathcal{F}$ if $K$ is a reducible element of $\mathbf{C}.$
\end{prop}
\begin{proof}
For any $K\in \mathbf{C}$, according to Proposition~\ref{P:Prop3}, $FL(K)=K.$ Hence $K\in \mathcal{F}.$

$(1)$ Suppose there exist $X, Y\in \mathcal{F}$ such that $K=X\cup Y,$ then there exist some elements $K_{i}(i\in I)$ and $K_{j}(j\in J)$ in $\mathbf{C}$ such that $K_{i}\subseteq X$, $K_{j}\subseteq Y$ and $X=\cup\{K_{i}\in \mathbf{C}|K_{i}\subseteq X(i\in I)\}, Y=\cup\{K_{j}\in \mathbf{C}|K_{j}\subseteq X(j\in J)\},$ where $I, J\subseteq \{1,2,\cdots, |\mathbf{C}|\}.$ Thus $K=(\cup_{K_{i}\subseteq X(i\in I)}K_{i})\cup(\cup_{K_{j}\subseteq Y(i\in J)}K_{j}).$ Since $K$ is an irreducible element of $\mathbf{C},$ then there exists $k\in I\cup
J$ such that $K=K_{k}$. Since $K=X\cup Y,$ then $X\subseteq K$ and $Y\subseteq K.$ If $K_{k}\subseteq X,$ then $K=X.$ If $K_{k}\subseteq Y,$ then $K=Y.$ Therefore, $K=X$ or $K=Y.$ So $K$ is a join-irreducible element of $\mathcal{F}.$

$(2)$ Since $K$ is a reducible element of $\mathbf{C}$, then there exist some elements in $\mathbf{C}-\{K\}$ such that $K$ is a union of these elements, i.e., there exist some elements $K_{i}(i\in I)$ such that $K=\cup_{i\in I}K_{i},$ where $I\subseteq  S=\{1,2,\cdots, |\mathbf{C}|\}.$  Since $K_{j(j\in S')}\subseteq \cup_{i\in S'}K_{i}$ for any $S'\subseteq S,$ then $FL(\cup_{i\in S'}K_{i})=\cup\{K_{j(j\in S')}|K_{j}\subseteq \cup_{i\in S'}K_{i}\}=\cup_{i\in S'}K_{i}.$ Therefore, for any $J, T\subseteq I,$ $FL(\cup_{j\in J}K_{j})=\cup_{j\in J}K_{j}$ and $FL(\cup_{t\in T}K_{t})=\cup_{t\in T}K_{t}.$  Hence $\cup_{j\in J}K_{j}\in \mathcal{F}, \cup_{t\in T}K_{t}\in \mathcal{F}.$ Thus there exist $J', T'\subseteq I$ such that $K=\cup_{i\in I}K_{i}=(\cup_{j\in J'}K_{j})\cup(\cup_{t\in T'}K_{t}).$ Therefore, $K$ is a join-reducible element of $\mathcal{F}$.
\end{proof}

The following theorem shows that the fixed point set of covering induced by any covering is a complete lattice.
\begin{theo}
Let $\mathbf{C}$ be a covering of $U$. 
$\mathcal{F}$ is a complete lattice.
\end{theo}
\begin{proof}
For any $\mathcal{S}\subseteq \mathcal{F},$ we need to prove that $\wedge \mathcal{S}\in \mathcal{F}$ and $\vee \mathcal{S}\in \mathcal{F}.$ In other words, we need to prove $\cup \mathcal{S}\in \mathcal{F}$ and $FL(\cap \mathcal{S})\in \mathcal{F}.$ Since $\mathcal{S}\subseteq \mathcal{F},$ then $\cap \mathcal{S}\subseteq U.$ According to Proposition~\ref{P:Prop3}, $FL(FL(\cap \mathcal{S}))=FL(\cap\mathcal{S}),$ i.e., $FL(\cap \mathcal{S})\in \mathcal{F}.$ If $\cup \mathcal{S}\notin \mathcal{F},$ then there exists $y\in \cup \mathcal{S}$ such that $K\nsubseteq \cup \mathcal{S}$ for any $K\in \mathbf{C}$ and $y\in K$.  Hence there exists $X\in \mathcal{S}$ such that $y\in X$ and $K\nsubseteq X$ for any $K\in \mathbf{C}$ and $y\in K$. So $y\notin FL(X),$ i.e., $X\notin \mathcal{F},$ which is  contradictory with $X\in \mathcal{F}.$
Therefore, $\cup\mathcal{S}\in \mathcal{F}.$
\end{proof}
The following example shows that the fixed point set of covering induced by any covering is not always a distributive lattice.
\begin{exa}
Let $U=\{1,2,3,4\}$ and $\mathbf{C}=\{\{1,2\},\{2, 3\},\{1,3,4\}\}.$
Then $\mathcal{F}=\{\emptyset, \{1,$\\$ 2\}, \{2, 3\}, \{1, 3, 4\}, \{1, 2, 3\}, U\}.$  $\{1, 2, 3\}\wedge (\{1, 3, 4\}\vee\{1, 2\})=\{1, 2, 3\},$ but $(\{1, 2,$\\$ 3\}\wedge \{1, 3, 4\})\vee(\{1, 2, 3\}\wedge\{1, 2\})=\{1, 2\}.$ Hence $\mathcal{F}$ is not a distributive lattice.
\end{exa}

In the following, we study under what condition the fixed point set of covering becomes a distributive lattice.
\begin{theo}\label{T:Distributive2}
Let $\mathbf{C}$ be a covering of $U$.  If $\mathbf{C}$ is unary, then
 $\mathcal{F}$ is a distributive lattice.
\end{theo}
\begin{proof}
For any $X, Y, Z\in \mathcal{F}$, there exist some elements $K_{i}(i\in I), K_{j}(j\in J)$ and $ K_{t}(t\in T)$ in $\mathbf{C}$ such that $K_{i}\subseteq X, K_{j}\subseteq Y, K_{t}\subseteq Z$ and $X=\cup\{K_{i}\in \mathbf{C}|K_{i}\subseteq X(i\in I)\}, Y=\cup\{K_{j}\in \mathbf{C}|K_{j}\subseteq Y(j\in J)\}, Z=\cup\{K_{t}\in \mathbf{C}|K_{t}\subseteq X(t\in T)\},$ where $I,$ $J,$ $T\subseteq \{1,2,\cdots, |\mathbf{C}|\}.$
$X\wedge (Y\vee Z)=FL(X\cap(Y\cup Z))=FL((X\cap Y)\cup(X\cap Z))=FL(((\cup_{K_{i}\subseteq X(i\in I)}K_{i})\cap(\cup_{K_{j}\subseteq Y(j\in J)}K_{j}))\cup((\cup_{K_{i}\subseteq X(i\in I)}K_{i})\cap(\cup_{K_{t}\subseteq Z(t\in T)}K_{t})))=FL((\cup_{ K_{i}\subseteq X(i\in I)\atop K_{j}\subseteq Y(j\in J)}(K_{i}\cap K_{j}))\cup(\cup_{ K_{i}\subseteq X(i\in I)\atop K_{t}\subseteq Z(t\in T)}(K_{i}\cap K_{t}))).$ $(X\wedge Y)\vee (X\wedge Z)=FL(X\cap Y)\cup FL(X\cap Z)=FL(\cup_{K_{i}\subseteq X(i\in I)\atop K_{j}\subseteq Y(j\in J)}(K_{i}\cap K_{j}))\cup FL(\cup_{ K_{i}\subseteq X(i\in I)\atop  K_{t}\subseteq Z(t\in T)}(K_{i}\cap K_{t})).$ Since $\mathbf{C}$ is unary, then $|Md(x)|=1$ for any $x\in U.$ Let $Md(x)=\{K_{x}\}$ for any $x\in U.$ According to Theorem~\ref{T:Theo1}, $K_{i}\cap K_{j}$ is a union of finite elements in $\mathbf{C}$. Hence, $K_{i}\cap K_{j}=\cup_{x\in K_{i}\cap K_{j}}K_{x}$. Therefore, $X\wedge (Y\vee Z)=FL(X\cap(Y\cup Z))={\cup_   {x\in (\cup_{ K_{i}\subseteq X(i\in I)\atop K_{j}\subseteq Y(j\in J)}(K_{i}\cap K_{j}))\cup(\cup_{ K_{i}\subseteq X(i\in I)\atop  K_{t}\subseteq Z(t\in T)}(K_{i}\cap K_{t}))}}K_{x}=(\cup_{y\in (\cup_{ K_{i}\subseteq X(i\in I)\atop K_{j}\subseteq Y(j\in J)}(K_{i}\cap K_{j}))}K_{y})\cup(\cup_{z\in (\cup_{K_{i}\subseteq X(i\in I)\atop  K_{t}\subseteq Z(t\in T)}(K_{i}\cap K_{t}))} K_{z})=(X\wedge Y)\vee (X\wedge Z).$ Hence $\mathcal{F}$ is a distributive lattice.
\end{proof}

The following proposition shows that an intersection of any two elements of the fixed point set of covering induced by a unary covering is closed.
\begin{prop}\label{P:Prop11}
Let $\mathbf{C}$ be a covering of $U$. If $\mathbf{C}$ is unary, then $X\cap Y\in \mathcal{F}$ for any $X, Y\in \mathcal{F}.$
\end{prop}
\begin{proof}
For any $y\in X\cap Y$, $y\in X $ and $y\in Y.$ Since $\mathbf{C}$ is unary, then $|Md(x)|=1$ for any $x\in U.$ Let $Md(x)=\{K_{x}\}$ for any $x\in U.$ Since $X, Y\in \mathcal{F},$ then $K_{y}\subseteq X$ and $K_{y}\subseteq Y,$ i.e., $K_{y}\subseteq X\cap Y.$ Therefore, $FL(X\cap Y)=\cup\{K\in \mathbf{C}|K\subseteq X\cap Y\}=\cup\{K_{y}|y\in X\cap Y\}=X\cap Y.$  Hence $X\cap Y\in \mathcal{F}.$
\end{proof}

The fixed point set of covering induced by a unary covering is both a pseudocomplemented lattice and a dual pseudocomplemented lattice.
That is to say any element of the fixed point set of covering has a pseudocomplement and a dual pseudocomplement.
For any element $X$, its pseudocomplement is the lower approximation of its complement. Its dual pseudocomplement is the union of
all the join-irreducible elements that contain any element in the complement of $X$.

\begin{theo}\label{T:Stone1}
Let $\mathbf{C}$ be a covering of $U$. If $\mathbf{C}$ is a unary, then:\\
$(1)$ 
$\mathcal{F}$ is a pseudocomplemented lattice, and $X^{\ast}=FL(X^{c})$ for any $X\in \mathcal{F};$\\
$(2)$ 
$\mathcal{F}$ is a dual pseudocomplemented lattice, and $X^{+}=\cup_{x\in X^{c}(x\in K\in \mathcal{J}(\mathcal{F}))}K$ for any $X\in \mathcal{F}.$
\end{theo}
\begin{proof}
$(1)$ According to Proposition~\ref{P:Prop3}, $FL(FL(X^{c}))=FL(X^{c})$,
then $FL(X^{c})\in \mathcal{F}.$ According to Proposition~\ref{P:Prop3}, $FL(X^{c})\subseteq X^{c}$. Hence $X\cap FL(X^{c})=\emptyset.$
Therefore, $FL(X\cap FL(X^{c}))=\emptyset.$
In the following, we need to prove $Y\subseteq FL(X^{c})$ if $FL(X\cap Y)=\emptyset$ for any $Y\in \mathcal{F}.$ According to Proposition~\ref{P:Prop11}, $X\cap Y\in \mathcal{F}$ for any $X, Y\in \mathcal{F}.$ Hence $FL(X\cap Y)=X\cap Y.$ If $FL(X\cap Y)=\emptyset$, then $X\cap Y=\emptyset.$
 Therefore, for any $Y\in \mathcal{F},$ if $FL(X\cap Y)=\emptyset,$ then $X\cap Y=\emptyset.$ Since $X\cap Y=\emptyset$, then $Y\subseteq X^{c}.$ According to Proposition~\ref{P:Prop3}, $FL(Y)\subseteq FL(X^{c}).$ Since $Y\in \mathcal{F},$ then $Y=FL(Y)\subseteq FL(X^{c}).$ Hence $X^{\ast}=FL(X^{c})$ for any $X\in \mathcal{F},$ i.e.,
$\mathcal{F}$ is a pseudocomplemented lattice.

$(2)$ 
For any $X\in \mathcal{F},$ $FL(\cup_{x\in X^{c}(x\in K\in \mathcal{J(\mathcal{F})})}K)=\cup_{x\in X^{c}(x\in K\in \mathcal{J(\mathcal{F})})}K.$ Hence $\cup_{x\in
X^{c}(x\in K\in \mathcal{J(\mathcal{F})})}K\in \mathcal{F}$ for any $X\in \mathcal{F}.$

It is straightforward that $X\cup(\cup_{x\in X^{c}(x\in K\in \mathcal{J(\mathcal{F})})}K)=U.$

 In the following, we need to prove that for any $Y\in \mathcal{F},$ if $X\cup Y=U,$ then $\cup_{x\in X^{c}(x\in K\in \mathcal{J(\mathcal{F})})}K\subseteq Y.$
 The following two cases are used to prove it. Case 1: If $\cup_{x\in X^{c}(x\in K\in \mathcal{J(\mathcal{F})})}K=X^{c},$ then $\cup_{x\in X^{c}(x\in K\in \mathcal{J(\mathcal{F})})}K\subseteq Y.$
 Case 2: If $X^{c}\subset\cup_{x\in X^{c}(x\in K\in \mathcal{J(\mathcal{F})})}K,$ then $X^{c}\subset Y.$ 
 If $X^{c}=Y,$ then $FL(X^{c})=FL(Y)=Y=X^{c}.$
 Since $FL(X^{c})=\cup\{K\in \mathbf{C}| K\subseteq X^{c}\}=\cup\{K_{x}\in Md(x)|x\in X^{c}\}
 =\cup\{K\in\mathcal{J}(\mathbf{C})|x\in X^{c}\wedge x\in K\}=\cup_{x\in X^{c}(x\in K\in \mathcal{J(\mathcal{F})})}K,$  $\cup_{x\in X^{c}(x\in K\in \mathcal{J(\mathcal{F})})}K=X^{c},$
 which is contradictory with $X^{c}\subset\cup_{x\in X^{c}(x\in K\in \mathcal{J(\mathcal{F})})}K.$
 Suppose $Y\subset\cup_{x\in X^{c}(x\in K\in \mathcal{J(\mathcal{F})})}K,$
 then there exists $y\in\cup_{x\in X^{c}(x\in K\in \mathcal{J(\mathcal{F})})}K$ such that $y\notin Y.$
 So $y\notin X^{c},$ which implies there exists $z\in X^{c}$ such that
 $y\in K$ for any $K\in \mathcal{J(\mathcal{F})}$ and $z\in K.$
 Since $X^{c}\subset Y,$ $z\in Y.$ So $K\nsubseteq Y$ for any $K\in \mathcal{J(\mathcal{F})}$ and $z\in K,$ i.e., $z\notin FL(Y).$ In other words, $FL(Y)\neq Y,$
 which is contradictory with $Y\in \mathcal{F}.$ Hence $\cup_{x\in X^{c}(x\in K\in \mathcal{J(\mathcal{F})})}K\subseteq Y.$
 Consequently, $X^{+}=\cup_{x\in X^{c}(x\in K\in \mathcal{J(\mathcal{F})})}K$ for any $X\in \mathcal{F}.$ Therefore,
$\mathcal{F}$ is a dual pseudocomplemented lattice.
\end{proof}

According to Theorem~\ref{T:Stone1}, the fixed point set of covering induced by a unary covering is both a pseudocomplemented lattice and a dual pseudocomplemented lattice. Moreover, according to Definition~\ref{D:Palgebra}, it is a double $p-$algebra.
\begin{remark}
Generally, the fixed point set of covering induced by any unary covering is neither a Stone algebra nor a dual Stone algebra. 
\end{remark}
\begin{exa}\label{E:Example3}
Let $U=\{1,2,3,4\}$ and $\mathbf{C}=\{\{3\},\{1\},\{1,3,4\},\{2,3\}\}.$
Then $\mathcal{F}=\{\emptyset,\{1\}, \{3\},\{1,3\},\{2,3\},\{1, 2, 3\},\{1,3,4\}, U\}.$ Let $X=\{3\}$ and $Y=\{2, 3\}.$ $X^{\ast}=FL(X^{c})=\{1\}, X^{\ast\ast}=FL((X^{\ast})^{c})=\{2,3\},$ i.e.,
$X^{\ast}\cup X^{\ast\ast}\neq U.$ Therefore,
$\mathcal{F}$ is not a Stone algebra. Similarly,
$Y^{+}=\cup_{x\in Y^{c}(x\in K\in \mathcal{J(\mathcal{F})})}K=\{1,3,4\}, Y^{++}=\cup_{x\in (Y^{+})^{c}(x\in K\in \mathcal{J(\mathcal{F})})}K=\{2,3\},$
i.e., $Y^{+}\cap Y^{++}\neq \emptyset.$ Therefore,
$\mathcal{F}$ is not a dual Stone algebra.
\end{exa}

According to Example~\ref{E:Example3}, the fixed point set of covering induced by any unary covering is not always a double Stone algebra. In the following, we study under what conditions that the fixed point set of covering induced by a covering is a boolean lattice and a double Stone algebra, respectively.
\begin{theo}\label{T:Boolean2}
Let $\mathbf{C}$ be a covering of $U.$ If $reduct(\mathbf{C})$ is a partition of $U,$ then 
$\mathcal{F}$ is a boolean lattice.
\end{theo}
\begin{proof}
First, we prove that if $reduct(\mathbf{C})$ is a partition, $\mathbf{C}$ is a unary covering. Suppose $\mathbf{C}$ is not a unary covering, then there exists $x\in U$ such that $|Md(x)|>1$.  So there exist $K_{1}, K_{2}\in reduct(\mathbf{C})$ such that $K_{1}, K_{2}\in Md(x),$ i.e., $x\in K_{1}\cap K_{2},$ which is contradictory with $ reduct(\mathbf{C})$ is a  partition of $U.$ Hence $\mathbf{C}$ is a unary covering.
According to Theorem~\ref{T:Distributive2},
 $\mathcal{F}$ is a distributive lattice. Moreover,
 $\mathcal{F}$ is a bounded lattice.

 Second, we need to prove only that 
 $\mathcal{F}$ is a complemented lattice. In other words,
 we need to prove that $X^{c}\in \mathcal{F}$ for any $X\in \mathcal{F}.$ Suppose $X^{c}\notin \mathcal{F},$
 then there exists $y\in X^{c}$ such that 
$K\nsubseteq X^{c}$ for any $K\in \mathbf{C}$ and $y\in K.$ Since $\mathbf{C}$ is a unary covering, then $Md(x)=\{K_{x}\}$ for any $x\in U.$ So $K_{y}\nsubseteq X^{c},$ i.e., there exists $x\in X$ such that $x\in K_{y}.$ Since $reduct(\mathbf{C})$ is a partition of $U,$ then $K_{x}=K_{y}.$ Therefore, $K_{x}\nsubseteq X,$ i.e., $x\notin FL(X).$ In other words, $FL(X)\neq X,$ which is contradictory with $
 X\in \mathcal{F}.$ Hence $X^{c}\in  \mathcal{F},$ i.e., $\mathcal{F}$ is a complemented lattice. Consequently,
 $\mathcal{F}$ is a boolean lattice.
\end{proof}

Theorem~\ref{T:Boolean2} shows that when the reduction of a covering is a partition of the universe, the fixed point set of covering induced by the covering is a boolean lattice. In the following, we prove that the fixed point set of covering is also a double Stone algebra.
\begin{theo}\label{T:DoubleStone2}
Let $\mathbf{C}$ be a covering of $U.$ If $reduct(\mathbf{C})$ is a partition of $U,$ then
$\mathcal{F}$ is a double Stone algebra.
\end{theo}
\begin{proof}
For any $X\in \mathcal{F}$, we prove $X^{\ast}=X^{c}=X^{+}.$ 
According to Theorem~\ref{T:Boolean2},  $\mathcal{F}$ is a boolean lattice. Hence $X^{c}\in \mathcal{F},$ i.e., $FL(X^{c})=X^{c}.$  So $X^{\ast}=FL(X^{c})=X^{c}.$ For any $y\in X^{c}$, if $K_{y}\in Md(y),$ then $K_{y}$ is a join-irreducible element of $\mathbf{C}$. According to Proposition~\ref{P:Prop10}, $K_{y}\in \mathcal{J}(\mathcal{F}).$ Suppose there exists $x\in X$ such that $x\in K_{y}.$ Since $reduct(\mathbf{C})$ is a partition of $U,$ then $K_{y}\in Md(x).$ Hence $K_{y}\nsubseteq X,$ i.e., $x\notin FL(X).$  So $FL(X)\neq X,$ which is contradictory with $X\in \mathcal{F}.$
Therefore, $X^{+}=\cup_{x\in X^{c}(x\in K\in \mathcal{J}(\mathcal{F}))}K=\cup_{x\in X^{c}(K\in Md(x))}K=X^{c}.$
Similarly, we can prove that $X^{\ast\ast}=X^{cc}=X=X^{++}.$
Therefore, $X^{\ast}\cup X^{\ast\ast}=U, X^{+}\cap X^{++}=\emptyset,$ i.e.,
$\mathcal{F}$ is both a Stone and a dual Stone algebra. Consequently,
$\mathcal{F}$ is a double Stone algebra.
\end{proof}

\section{Conclusions}
In this paper, we established two types of partially ordered sets by the fixed points of the lower approximations of
the first and sixth types of covering-based rough sets, respectively. one 
is called the fixed point set of neighborhoods,
the other 
is called the fixed point set of covering. Both of them are lattices, where the least upper bound of any two
elements of the fixed point set of neighborhoods is the join of these two elements and
the greatest lower bound is the intersection of these two elements.
In the fixed point set of covering, the greatest lower bound of any two elements is the lower
approximation of the intersection of these two elements. For any covering, we proved that the fixed point set of neighborhoods is
 both a complete and a distributive lattice. It is also a double $p-$algebra. Especially, when the neighborhoods form a partition of
 the universe, the fixed point set of neighborhoods is both a boolean lattice and a double Stone algebra.
  Similarly, for any covering, the fixed point set of covering is a complete lattice. When a covering is unary, the fixed point set of covering induced by this covering is also both a distributive lattice and a double $p-$algebra.
 The fixed point set of covering is both a boolean lattice and a double Stone algebra when the reduction of the covering forms a partition of the universe.\\

\section{Acknowledgments}
This work is supported in part by the National Natural Science Foundation
of China under Grant No. 61170128, the Natural Science Foundation of Fujian
Province, China, under Grant Nos. 2011J01374 and 2012J01294, and the Science
and Technology Key Project of Fujian Province, China, under Grant No.
2012H0043.



\end{document}